\documentclass[a4paper,11pt,twoside]{article}
\usepackage{amsmath,amsthm,amsfonts,amssymb,latexsym}
\usepackage{fancyhdr,graphicx}
\usepackage{algpseudocode}
\usepackage{algorithm}
\usepackage{hhline}

\usepackage{url}
\usepackage[justification=centering]{caption}

\usepackage{multirow}

\usepackage{bigstrut}
\usepackage{mathtools}
\usepackage{graphicx}

\usepackage{subfigure}

\newtheorem{Theorem}{Theorem}

\fancypagestyle{plain}
{\small INTERNATIONAL JOURNAL OF COMPUTERS COMMUNICATIONS $\&$ CONTROL}

\begin{document}

\title{Weighted Random Search for Hyperparameter Optimization}
\author{A.C. Florea, R. Andonie} 
\date{}
\maketitle

{\small
\begin{flushleft}
\textbf{Adrian-C\u at\u alin Florea}\\
Department of Electronics and Computers\\
Transilvania University of Bra\c sov, Romania\\
acflorea@unitbv.ro
\end{flushleft}

{\small
\begin{flushleft}
\textbf{R\u azvan Andonie}\\
Department of Computer Science\\
Central Washington University, USA\\
andonie@cwu.edu
\end{flushleft}

\begin{abstract}

We introduce an improved version of Random Search (RS), used here for hyperparameter optimization of machine learning algorithms. 
Unlike the standard RS, which generates for each trial new values for all hyperparameters, we generate new values for each hyperparameter with a  probability of change. The intuition behind our approach is that a value that already triggered a good result is a good candidate for the next step, and should be tested in new combinations of hyperparameter values. Within the same computational budget, our method yields better results than the standard RS. Our theoretical results prove this statement. We test our method on a variation of one of the most commonly used objective function for this class of problems (the Grievank function) and for the hyperparameter optimization of a deep learning  CNN architecture. Our results can be generalized to any optimization problem defined on a discrete domain.

\end{abstract}

\section{Introduction} \label{introduction}

The vast majority of machine learning algorithms involve two different sets of parameters: the training parameters and the  meta-parameters (also known as \textit{hyperparameters}). While the training parameters are learned during the training phase, the values of the hyperparameters have to be specified before the learning phase. For instance, the hyperparameters of neural networks typically specify the architecture of the network (number and type of layers, number and type of nodes, etc).

Determining the optimal combination of hyperparameter values leading to the best generalization performance can be done through repeated training and evaluation sessions, trying different combinations of hyperparameter values. We call each training + evaluation process for one combination of hyperparameter values a \textit{trial}. Each trial is computationally expensive, since it involves re-training the model. In addition, the number of trials increases generally exponential with the number of hyperparameters. Therefore, it is important to reduce the number or trials \cite{florea:hal-01821037}. This can be done by both reducing the number of hyperparameters and reducing the value range of each hyperparameter, while still maximizing the probability to hit the optimal combination \cite{conf/nips/BergstraBBK11,journals/jmlr/BergstraB12}.

Various hyperparameter optimization methods were developed during the years, ranging from very simple ones, such as Grid Search (GS) and manual tuning \cite{LeCun2012,Hinton2012,Smusz2015}\footnote{https://github.com/jaak-s/nips2014-survey - 82 out of 86 optimization related papers presented at the NIPS 2014 conference used GS.}, to highly elaborated techniques: Nelder-Mead  \cite{Albelwi2017, NelderMead65}, simulated annealing \cite{Kirkpatrick1984}, evolutionary algorithms  \cite{Hansen:2003:RTC:772374.772376}, Bayesian methods \cite{Thornton:2013:ACS:2487575.2487629}, etc.

Recently, there has been significant interest in the area of hyperparameter optimization, especially since the rise of deep learning which puts a lot of pressure on the existing techniques due to the very large number of hyperparameters involved and the significant training time needed for such architectures. The focus in hyperparameter optimizations presently oscillates between introducing more sophisticated techniques (Sequential Model-Based Global Optimization \cite{conf/nips/BergstraBBK11}, reinforcement learning \cite{DBLP:journals/corr/ZophL16,DBLP:journals/corr/ZophVSL17}, etc) and various attempts to optimize existing simple techniques. 

RS falls into the category of simple algorithms \cite{conf/nips/BergstraBBK11,journals/jmlr/BergstraB12}. Making use of the same computational budget, RS generally yields better results than GS or more complicated hyperparameter optimization methods \cite{conf/nips/BergstraBBK11}. Especially in higher dimensional spaces, the computation resources required by RS methods are significantly lower than for GS \cite{Lemley2016}. RS consists in drawing samples from the parameter space following a particular distribution for each of the parameters. Each trial is drawn and evaluated independently from the others, which makes RS a very good candidate for parallel implementations. 

Some recent attempts to optimize the RS algorithm are: Li's \textit{et al.} Hyperband \cite{DBLP:journals/corr/LiJDRT16}, which speeds up RS through adaptive resource allocation and early-stopping; Domhan \textit{et al.} \cite{Domhan:2015:SUA:2832581.2832731}, which have developed a probabilistic model to mimic early termination of sub-optimal candidate; and Florea \textit{et al.} \cite{florea:hal-01821037}, where we introduced a dynamically computed stopping criterion for RS, reducing the number of trials without reducing the generalization performance.

There are various software libraries implementing hyperparameter optimization methods. Hyperopt \cite{1749-4699-8-1-014008} and Optunity  \cite{DBLP:journals/corr/ClaesenSPMM14} are currently two of the most advanced standalone packages. Bayesian techniques are implemented by packages like BayesianOptimization \cite{NIPS2012_4522} and pyGPGO \cite{Shahriari2016TakingTH}. Some of the best known general purpose machine learning software libraries also provide hyperparameter optimization: LIBSVM \cite{CC01a} and scikit-learn \cite{scikit-learn} come with their own implementation of GS, with scikit-learn also offering support for RS. Auto-WEKA \cite{JMLR:v18:16-261}, built on top of Weka  \cite{Hall:2009:WDM:1656274.1656278} is able to perform GS, RS, and Bayesian optimization.

Lately, commercial cloud-based services started to offer hyperparameter optimization capabilities. Among them we count Google HyperTune \cite{Google-HyperTune}, BigML's OptiML \cite{BigML}, and SigOpt \cite{SIGOPT}. All of them support mixed search domains, SigOpt being able to handle multi-objective, multi-solution, constraint (linear and black-box), and parallel optimization.

In this context, our contribution is an improved version of the RS method, the \textit{Weighted Random Search} (WRS) method. Unlike the standard RS, which generates for each trial new values for all hyperparameters, we generate new values for each hyperparameter with a  probability of change $p$ and we use the best value found so far for that particular hyperparameter with probability $1-p$, where $p$ is proportional to the hyperparameter's relative importance in the variation of the objective function. The intuition behind our approach is that a value that already triggered a good result is a good candidate for a new trial and should be tested in new combinations of hyperparameter values.

For the same number of trials, the WRS algorithm produces significantly better results than RS. We obtained theoretical results which prove this statement. We  tested our algorithm on a slightly modified version of one of the most commonly used objective function for this class of problems - the Grievank \cite{Griewank81} function, as well as for the hyperparameter optimization of a deep learning CNN architecture using the CIFAR-10 \cite{cifar-10} dataset. 

Unlike our previous work on RS optimization \cite{florea:hal-01821037}, where our focus was on the dynamic reduction of the number of  trials, the focus of the WRS method is the optimization of the classification (prediction) performance  within the same computational budget. The two approaches make use of different optimization techniques.

The paper proceeds as follows. Section \ref{new_approach} is a general presentation of our WRS algorithm. Section \ref{theoretical_aspects} describes theoretical results and the convergence of WRS. Sections \ref{grievank_function_optimization} and \ref{cnn_hyperparameter_optimization} contain experimental results. The paper is concluded with Section \ref{conclusion}.

\section{The WRS Method} \label{new_approach}

We  first present the generic intuitive description of the WRS algorithm, which is the core of our contribution. Technical details will be provided later.

The standard RS technique \cite{conf/nips/BergstraBBK11} generates a new multi-dimensional sample at each step $k$, with new random values for each of the sample's dimensions - features, in our case -  $X^k = \{x_i^k\}, i=1, \dots,d$, where $x_i$ is generated according to a probability distribution $P_i(x), i=1, \dots,d$, and $d$ is the number of dimensions.

WRS is an improved version of RS, designed for hyperparameter optimization. It assigns probabilities of change $p_{i}, i=1, \dots,d$ to each dimension. For each dimension $i$, after a certain number of steps $k_i$, instead of always generating a new value, we generate it with probability $p_i$ and use the best value known so far with probability $1-p_i$. 

The intuition behind the proposed algorithm is that after already fixing $d_0$ ($1<d_0<d$) values, each $d$-dimensional optimization problem reduces itself to a  $d-d_0$ dimensional one. In the context of this $d-d_0$ dimensional problem, choosing a set of values that already performed well for the remaining dimensions might prove more fruitful than choosing some $d-d_0$ random values. In order to avoid getting stuck in a local optimum, instead of setting a hard boundary between choosing the best combination of values found so far or generating new random samples, we assign probabilities of change for each dimension of the search space.

WRS has two phases. In the first phase, it runs the RS for a predefined number of trials and allows: \textit{a)} to identify the best combination of values so far; and \textit{b)} to give enough input on the importance of each dimension in the optimization process. The second phase considers the probabilities of change and generates the candidate values according to them. Between these two phases, we run one instance of fANOVA \cite{HutHooLey14}, in order to determine the importance of each dimension with respect to the objective function. Intuitively, the most important dimension (the dimension that yields the largest variation of the objective function) is the one that should change most frequently, to cover as much of the variation range as possible. For a dimension with small variation of the objective function, it might be more efficient to keep a certain temporary optimum value once this has been identified.

A step of the WRS algorithm applied to function maximization is described by Algorithm \ref{alg:wrs_one_step}, whereas the entire method is detailed in Algorithm \ref{alg:wrs_global}. $F$  is the objective function, the value $F(X)$ has to be computed for each argument, $X^k$ is the best argument at iteration $k$, whereas $N$ is the total number of iterations. 

At each step of Algorithm \ref{alg:wrs_global}, at least one dimension will change,  hence we always choose at least one of the $p_i$ probabilities to be equal to one. For the other probabilities, any value in $(0,1]$ is valid. If all values are one, then we are in the case of RS. 

 \begin{algorithm}
 \caption{A WRS Step - Objective Function Maximization}
 \label{alg:wrs_one_step}
 \begin{algorithmic}[1]
 \renewcommand{\algorithmicrequire}{\textbf{Input:}}
 \renewcommand{\algorithmicensure}{\textbf{Output:}}
 \Require $F$;  ($X^k$, $F(X^k)$);  $p_{i},k_{i},P_{i}(x) ,i=1,\dots,d$
 \Ensure  ($X^{k+1},F(X^{k+1})$)
 \State Randomly generate  $p$, uniform in (0,1) 
 \For {$i = 1$ to $d$}
 	\If {($p_i \geq p \text{ or } k \leq k_i $)}
 	    \State \textit{// either the probability condition is met or more samples are needed}
 		\State Generate $x^{k+1}_i$ according to $P_i(x)$
 	\Else
 	 	\State $x^{k+1}_i = x^{k}_i$
	 \EndIf
 \EndFor
  \\ \textit{// usually this is the most time consuming step}
  \State Compute $F(X^{k+1})$
  \If {$F(X^{k+1}) \geq F(X^{k}$)}
  	\State
  	\Return $(X^{k+1}, F(X^{k+1}))$
  \Else
    	\State
  	\Return $(X^{k}, F(X^{k}))$
 \EndIf   
 \end{algorithmic} 
 \end{algorithm}

 \begin{algorithm}
 \caption{WRS - Objective Function Maximization}
 \label{alg:wrs_global}
 \begin{algorithmic}[1]
 \renewcommand{\algorithmicrequire}{\textbf{Input:}}
 \renewcommand{\algorithmicensure}{\textbf{Output:}}
 \Require $F$;  $N$;  $P_{i}(x),i=1,\dots,d$
 \Ensure  ($X^{N},F(X^{N})$)
 \ \\ \textit{// Phase 1 - Run RS}
 \For {$k = 1$ to $N_0<N$}
	\State Perform RS step, compute $(X^{k}, F(X^{k}))$
 \EndFor
 \ \\ \textit{// Intermediate phase, determine input for WRS}
 \State Determine the probability of change $p_{i},i=1,\dots,d$
 \State Determine the minimum number of required values $k_{i},i=1,\dots,d$
 \ \\ \textit{// Phase 2 - Run WRS}
 \For {$k = N_0+1$ to $N$}
	\State Perform WRS Step described in Algorithm \ref{alg:wrs_one_step}, compute $(X^{k}, F(X^{k}))$
 \EndFor
 \State
 \Return ($X^{N},F(X^{N})$)
 \end{algorithmic} 
 \end{algorithm}

Besides a way to compute the objective function, Algorithm \ref{alg:wrs_one_step} requires only the combination of values that yields the best $F(X)$ value obtained so far and the probability of change for each dimension. The current optimal value of the objective function can be made optional, since the comparison can be done outside of Algorithm \ref{alg:wrs_one_step}. Algorithm \ref{alg:wrs_global} coordinates the sequence of the described steps and calls Algorithm \ref{alg:wrs_one_step} in a loop, until the maximum number of trials $N$ is reached.

\section{Theoretical Aspects and Convergence} \label{theoretical_aspects}

We aim to analyze the theoretical convergence of Algorithm \ref{alg:wrs_global} and compare it to the RS method. Similar to GS and RS, we make the assumption that there is no statistical correlation between the variables of the objective function (hyperparameters). To make explanations more intuitive, we first discuss the two-dimensional case, and then generalize for the multi-dimensional case. We will also define what we consider "a set of good candidate values" for $p_{i}$ and $k_{i},i=1,\dots,d$ (used in steps 6 \& 7, Algorithm \ref{alg:wrs_global}). We denote by $n\geq1$ the number of iterations (steps) for both RS and WRS.

\subsection{Two-dimensional case} \label{bi-dimensional-case}

In the two-dimensional case  ($d=2$), we aim to maximize a function $F:{\text{S}_1}\times {\text{S}_2}\to {\mathbb {R}}$, where $\text{S}_1$ and $\text{S}_2$ are countable sets. We define as \textit{global optimum} the point $X^*(x_1^*, x_2^*)$, with $x_1^*\in\text{S}_1$ and $x_2^*\in\text{S}_2$, so that $F(X^*)\geq{F(X)}, \forall X \in\text{S}_1\times\text{S}_2$. $p_i,k_i, (i=1,2)$ are the probabilities of change, respectively, the required number of distinct values for $x_i$, as previously defined. $|S_i|$ is the cardinality of $S_i,i=1,2$. We denote by $p_{RS:n}$ and $p_{WRS:n}$ the probability for RS, respectively WRS, to reach the global optimum after $n$ steps. 

The following theorem establishes that, in the two-dimensional case we can choose $k_2$ so that 
\begin{equation}\label{eq:WRSvsRS}
p_{WRS:n}\geq{p_{RS:n}} 
\end{equation}

\begin{Theorem}
For any function $F:{\text{S}_1}\times {\text{S}_2}\to {\mathbb {R}}$ there exists $k_2$, so that $p_{WRS:n}\geq{p_{RS:n}}$.
\end{Theorem}

\begin{proof}
We consider the case of maximizing function $F$, and choose the arguments in the decreasing order of their probabilities of change. Since the value for one dimension always changes, we have $p_1=1,p_2\leq 1$. Having $p_1=1$, the value of $k_1$ can be ignored: the condition at step 3 in Algorithm \ref{alg:wrs_one_step} will be always true for $i=1$.

At each step  $k$, $k\leq{k_2}$, WRS is identical with RS and we have $p_{WRS:k}={p_{RS:k}}$. At step $k+1 > k_2$, RS generates new values for $x^{k+1}_1$ and $x^{k+1}_2$, and computes  $F(x^{k+1}_1,x^{k+1}_2)$. For WRS, $x^{k+1}_1$ is generated with probability one, but $x^{k+1}_2$ is generated with $p_2 \leq 1$. With probability $1-p_2$, the best value known so far for $x_2$ is used, instead of generating a new one. $X^{k+1}$ can be written as:

\begin{equation}
X^{k+1} = \begin{cases} (x^{k+1}_1, x^{k+1}_2), & \mbox{with probability } p_2 \\ (x^{k+1}_1, x^{k}_2), & \mbox{with probability } 1-p_2 \end{cases}
\end{equation}

With probability $p_2$, each step in WRS is identical to the same step in RS, and all points in ${\text{S}_1}\times {\text{S}_2}$ are accessible to WRS. Therefore, RS and WRS have the same search space and both converge probabilistically to the global optimum. 

Ignoring the statistical correlation between the two variables, the probability of RS to hit the optimum after one iteration (the best case) is:

\begin{equation}\label{eq:prs}
p_{RS}=\frac{1}{|S_1|}\frac{1}{|S_2|}
\end{equation}

For WRS, this probability is:  

\begin{equation}\label{eq:pwrs}
p_{WRS}=\frac{1}{|S_1|} \Big(p_2\frac{1}{|S_2|} + (1-p_2)\frac{1}{|S_2|-m_2+1}\Big)
\end{equation}

where $m_2$ is the number of distinct values already generated for $x_2$. 

Using \eqref{eq:prs} and \eqref{eq:pwrs}, \eqref{eq:WRSvsRS} becomes:

\begin{equation}\label{eq:WRSvsRS-explicit}
1- (1 - p_{WRS})^n \geq 1-(1-p_{RS})^n
\end{equation}

which is equivalent to

\begin{equation}\label{eq:WRSvsRS-explicit-onestep}
\frac{1}{|S_1|} \Big(p_2\frac{1}{|S_2|} + (1-p_2)\frac{1}{|S_2|-m_2+1}\Big) \geq \frac{1}{|S_1|} \frac{1}{|S_2|}
\end{equation}

After multiplying both sides by $|S_1|$, \eqref{eq:WRSvsRS-explicit-onestep} can be rewritten as

\begin{equation}\label{eq:WRSvsRS-explicit-onestep-n}
p_2\frac{1}{|S_2|} + (1-p_2)\frac{1}{|S_2|-m_2+1} \geq \frac{1}{|S_2|}
\end{equation}

which reduces to
\begin{equation}\label{eq:WRSvsRS-explicit-onestep-simplified}
p_2(1 -m_2) \geq 1 - m_2
\end{equation}

Because $p_2\leq 1$, \eqref{eq:WRSvsRS-explicit-onestep-simplified} is true for $m_2>1$. Relation \eqref{eq:WRSvsRS} is therefore satisfied if we choose $k_2$ so that at least two distinct values are generated for $x_2$.

\end{proof}

\subsection{Multi-dimensional case}

For the general case of optimizing a function $F:{\text{S}_1}\times {\text{S}_2} \: \dots \: \times {\text{S}_d}\to {\mathbb {R}}$, with $S_i, i=1,\dots,d$ countable sets and under the same assumption that the variables are not statistically correlated, $P_{RS}$  and $P_{WRS}$ are defined as:

\begin{equation}\label{eq:genericPRS}
p_{RS}=\prod_{i=1}^d\frac{1}{|S_i|}
\end{equation}

\begin{equation}\label{eq:genericPWRS}
p_{WRS}=\frac{1}{|S_1|} \prod_{i=2}^d\Big(p_i\frac{1}{|S_i|} + (1-p_i)\frac{1}{|S_i|-m_i+1}\Big)
\end{equation}

where $m_i$ is the number of distinct values already generated for $x_i$.

Following the rationale from Section \ref{bi-dimensional-case}, we have the following theorem:

\begin{Theorem}
For any function $F:{\text{S}_1}\times {\text{S}_2} \: \dots \: \times {\text{S}_d}\to {\mathbb {R}}$ there exist $k_i, i=1,\dots,d$, so that  $p_{WRS:n}\geq{p_{RS:n}}$.
\end{Theorem}

\begin{proof}

We consider again the maximization of function $F$. 

Given $k_i, i=1, \dots,d$ the minimum number of values required for each of the dimensions $x_i$ with $k_i\leq k_{i+1}, i=1, \dots,d-1$ and $k\geq k_d$, $X_{k+1}$ is given by:

\begin{equation}
\begin{cases} 
(x^{k+1}_1, x^{k+1}_2,\dots,x^{k+1}_{d-1},x^{k+1}_d),& \mbox{with probability } p_d 
\\ (x^{k+1}_1, x^{k+1}_2,\dots,x^{k+1}_{d-1},x^{k}_d),& \mbox{with probability } p_{d-1}-p_d
\\ \dots
\\ (x^{k+1}_1, x^{k+1}_2,\dots,x^{k}_{d-1},x^{k}_d),& \mbox{with probability }  p_2-p_3
\\ (x^{k+1}_1, x^{k}_2,\dots,x^{k}_{d-1},x^{k}_d), & \mbox{with probability } 1-p_2
\end{cases}
\end{equation}

Starting from \eqref{eq:genericPRS} and  \eqref{eq:genericPWRS}, we can express $P_{WRS:n}$ as:
\begin{equation}\label{eq:genericPWRSatn}
1 - \bigg(1-\frac{1}{|S_1|} \prod_{i=2}^d\Big(p_i\frac{1}{|S_i|} + (1-p_i)\frac{1}{|S_i|-m_i+1}\Big)\bigg)^n
\end{equation}
and $P_{RS:n}$ as:
\begin{equation}\label{eq:genericPRSatn}
1-\Big(1-\prod_{i=1}^d\frac{1}{|S_i|}\Big)^n
\end{equation}

Since all elements of the products from \eqref{eq:genericPWRSatn} and \eqref{eq:genericPRSatn} are positive ($1-p_i\geq 0$, and $m_i$ cannot be greater than $|S_i|$), a sufficient condition to satisfy \eqref{eq:WRSvsRS} is:

\begin{equation}\label{eq:PRSTermvsPWRSTerm}
\Big(\frac{1}{|S_i|} + (1-p_i)\frac{1}{|S_i|-m_i+1}\Big) \geq \frac{1}{|S_i|}
\end{equation}
for each $i\geq2$), which reduces to
\begin{equation}\label{eq:WRSvsRS-explicit-onestep-simplified-n}
p_i(1 -m_i) \geq 1 - m_i
\end{equation}

and, since $p_i\leq 1$, is equivalent with

\begin{equation}\label{eq:convergence_condition}
m_i\geq2, \textrm{for } i=2,\dots,d
\end{equation}

Relation \eqref{eq:WRSvsRS} is satisfied if we choose $k_i$ so that at least two distinct values are generated for each dimension.

\end{proof}

According to these results, for a well chosen set of $k_i,i=1,\dots,d$, at any step $n$, WRS has a greater probability than RS to find the global optimum. Therefore, given the same number of iterations, on average, WRS finds the global optimum faster than RS. In other words, on average, WRS converges faster than RS. 

Moreover, for WRS, the number of generated values for $x_i, i=1,\dots,d$, follows a binomial distribution with probability $p_i$.  After $n$ steps, the expected value for this distribution is $np_i$. Therefore, $m_i$ has, on average, an upper bound of $np_i$. The number of distinct generated values depends on the cardinality of $\text{S}_i$ and the probability distribution used to generate $x_i$.

For example, in the case of the uniform distribution, the expected value for $m_i$ is:

\begin{equation}\label{eq:distinct_values}
E[m_i] = \sum_{1}^{|S_i|}{\left(1-{\left(\frac{|Si|-1}{|Si|}\right)}^{np_i}\right)}
\end{equation}

and $m_i>1$ when $np_i>1$. Hence, for any number of steps $n$, with $n\geq1/p_i$, \eqref{eq:WRSvsRS} is true. By choosing $k_i$ so that $k_i>1/p_i$, \eqref{eq:WRSvsRS} is true for all values of $n$. It can be also observed that the difference between $p_{WRS:n}$ and $p_{RS:n}$ increases with an increasing value of $n$. 

\subsection{Choosing $p_i$ and $k_i$}

Regardless of the distribution used for generating $x_i$, by choosing for $k_i$ (step 6, Algorithm \ref{alg:wrs_global}) a value that can guarantee the generation of at least two distinct samples, \eqref{eq:WRSvsRS} is true and WRS has a higher  probability to find the optimum than RS. 

We decide to sort the function variables depending on their importance (weight) and assign their probabilities $p_i$ accordingly: the smaller the weight of a parameter, the smaller it's probability of change. Therefore, the most important parameter is the one that will always change ($p_1=1$). In order to compute the weight of each parameter, we run RS for a predefined number of steps, $N_0<N$. On the obtained values, we apply fANOVA \cite{HutHooLey14} to estimate the importance of the hyperparameters. If $w_i$ is the weight of the $i$-th parameter and $w_1$ is the weight of the most important one, then $p_i = w_i/w_1, i=1,\dots,d$. 

By assigning higher probabilities of change  to the most important parameters and running RS for $N_0$ steps, we make sure that \eqref{eq:convergence_condition} is satisfied for these parameters. For simplicity, we set $k_i=N_0$ for all parameters, but these values can be adjusted depending on the objective function.

\section{An Example: Griewank Function Optimization} \label{grievank_function_optimization}

To illustrate the concept behind WRS, we consider a simple function with a known analytic form. Since the function is very fast to compute, we can test the performance of our algorithm on a very large number of runs. This will allows us to perform an unpaired t-test on the results and rule out the random factor when assessing its performance.

The Grievank \cite{Griewank81} function is widely used to test the convergence of optimization algorithms. It's analytic form is given by:

\begin{equation}\label{eq:grievank_n}
G_d = 1 + \frac{1}{4000}\sum_{i=1}^dx_i^2-\prod_{i=1}^d\cos{\frac{x_i}{\sqrt{i}}}
\end{equation} 

The function poses a lot of stress on optimization algorithms due to its very large number of local minimums.  We use a slightly modified version of $G_6$, given by:

\begin{equation}\label{eq:grievank_6}
G^*_6 = 1 + \frac{i-1}{4000}\sum_{i=1}^6x_i^2-\prod_{i=1}^6\cos{\frac{x_i}{\sqrt{i}}}
\end{equation}

and maximize $-G^*_6$. The function has a global maximum at $0$, for $x_i=0, i=1,\dots,6$. The term $i-1$ is introduced in order to alter the parameters' importance(weight) which, otherwise, would have been the same across all dimensions. 
We use $S=[-600, 600]$ for all six parameters and run the optimizer for 1000 trials, with an initial RS phase of $1000/e=368$ steps \cite{florea:hal-01821037}. After the first RS phase, we run fANOVA and obtain the weights of the parameters, listed along with their probabilities of change in Table \ref{table:weights-grievank}.

\begin{table}[h!]
\caption{Parameter weights and probabilities for $G_6^*$}
\begin{center}
\begin{tabular}{ |c|c|c|c|c|c|c| }
\hline
Parameter & $x_1$ & $x_2$ & $x_3$ & $x_4$ & $x_5$ & $x_6$ \\
\hline
Weight & 0.07 & 0.18 & 1.24 & 7.77 & 23.52 & 43.96 \\
\hline
Probability & 0.002 & 0.004 & 0.028 & 0.177 & 0.535 & 1.00 \\
\hline
\end{tabular}
\label{table:weights-grievank}
\end{center}
\end{table}

We compare our results against RS, on the same search space, performing 1000 trials on 10000 runs. Table \ref{table:RSvsWRS-grievank} shows the best result achieved by both RS and WRS across all 10000 runs, as well as the average value and the standard deviation of the achieved results across all runs. The standard error for the t-test is $0.176$, $df=19998$ and $\text{P-value}\leq 0.001$.

\begin{table}[h!]
\caption{WRS vs. RS results for $G_6^*$ - values for 1000 runs}
\begin{center}
\begin{tabular}{ |c|c|c|c| }
\hline
Optimizer &  Best Found Value & Average Value & SD \\
\hline
RS & -1.50 & -33.10 & 14.06  \\
\hline
WRS & -\textbf{1.28} & \textbf{-14.58} & \textbf{10.63}  \\
\hline
\end{tabular}
\label{table:RSvsWRS-grievank}
\end{center}
\end{table}

The results obtained by WRS are clearly better than the ones achieved by RS, as also depicted in Fig. \ref{fig:grievank-WRSvsRS}.

Fig. \ref{fig:grievank-WRSConvergence} shows the results obtained for one optimization session with 1000 trials. It can be observed that the algorithm tends to achieve improving results as the number of trials increases.

\begin{figure*}[htb!]
\centerline{\includegraphics[width=0.65\textwidth]{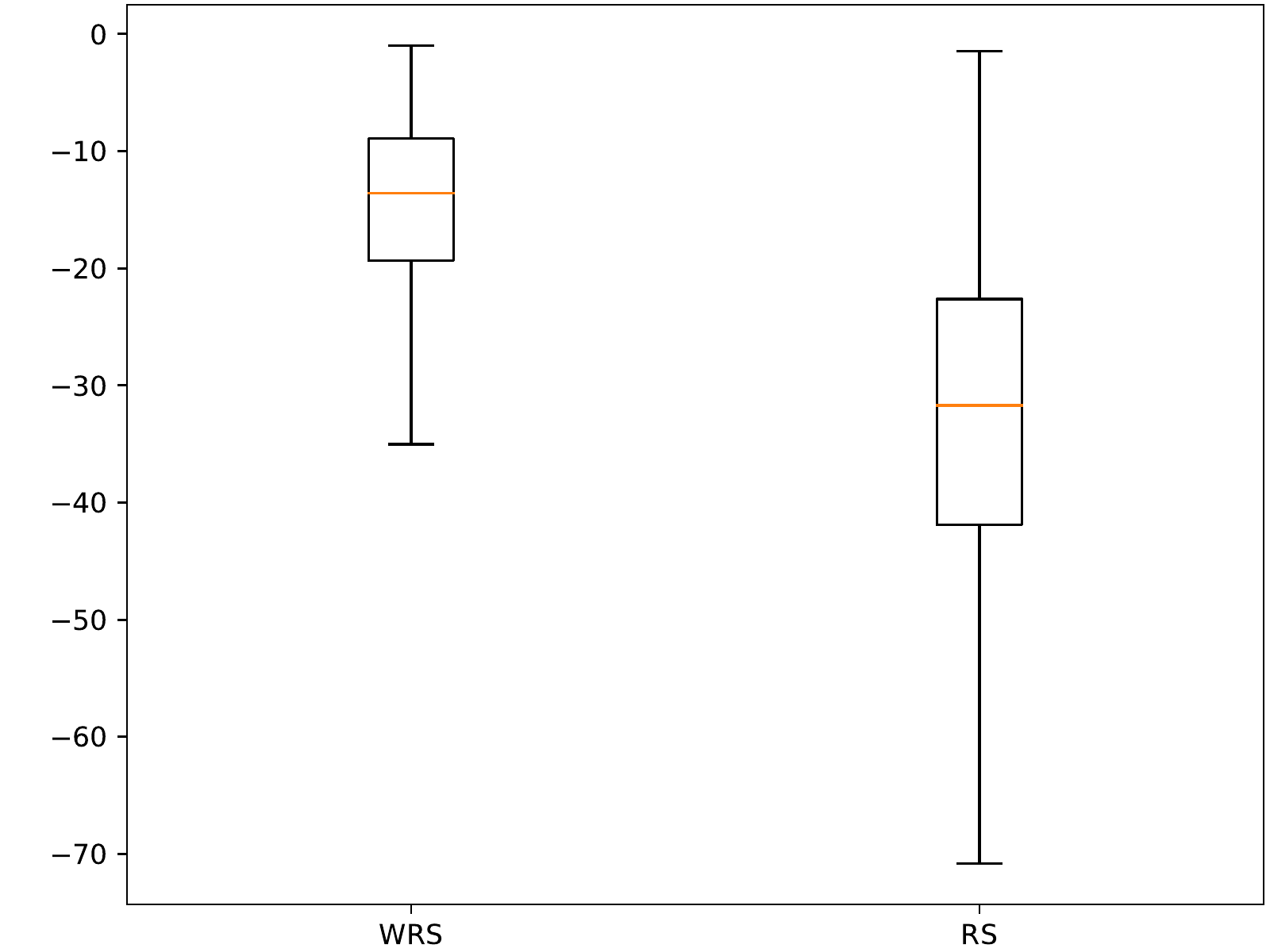}}
\caption{Performance of WRS vs. RS for the $G_6^*$ optimization}
\label{fig:grievank-WRSvsRS}
\end{figure*}

\begin{figure*}[htb!]
\centerline{\includegraphics[width=0.65\textwidth]{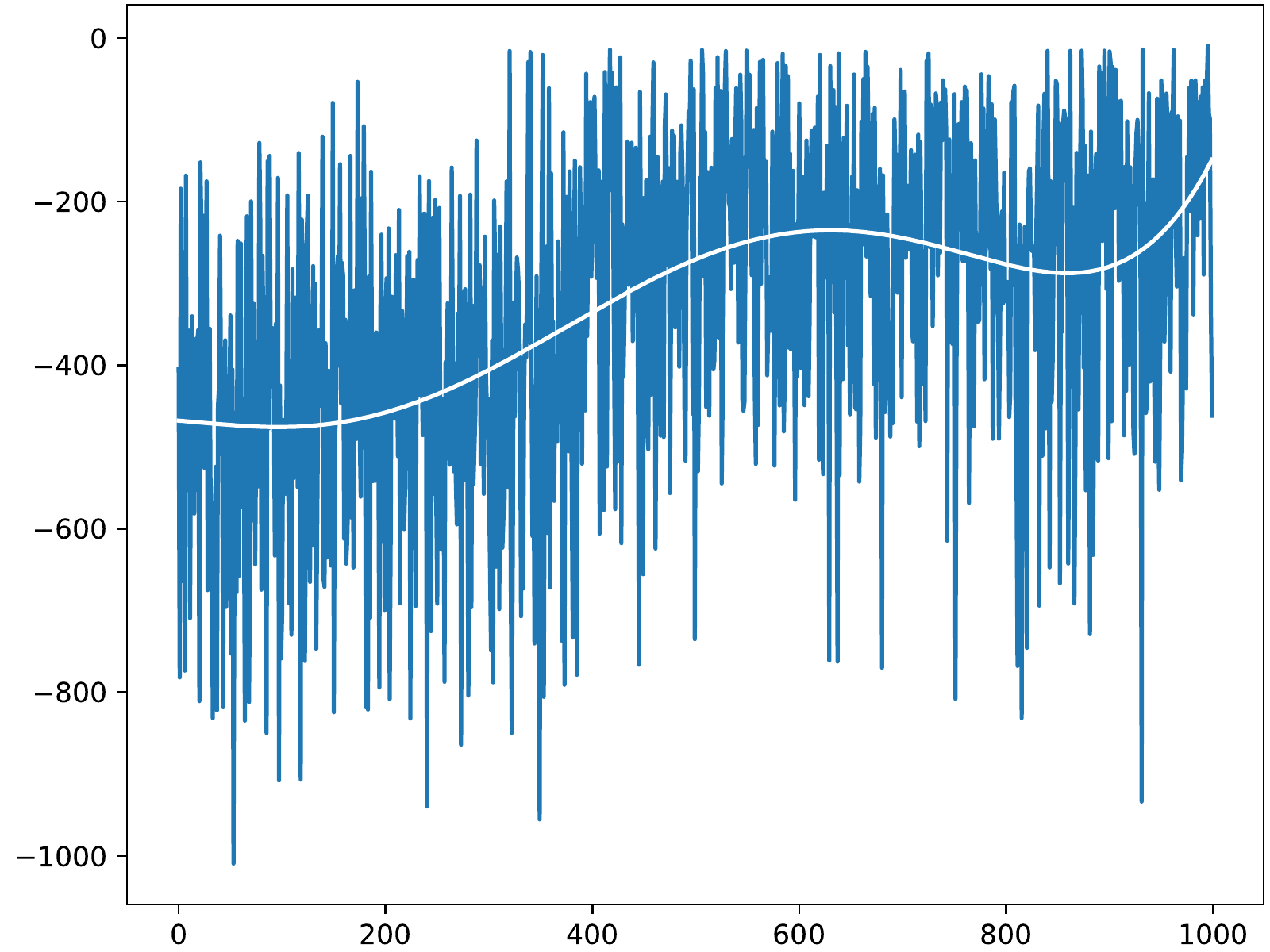}}
\caption{Convergence of WRS for the $G_6^*$ function}
\label{fig:grievank-WRSConvergence}
\end{figure*}

\section{CNN Hyperparameter Optimization} \label{cnn_hyperparameter_optimization}

Our next application of the WRS is for the optimization of a CNN architecture. Currently CNN in one of the best and most used tools for image recognition and machine vision \cite{Patterson:2017:DLP:3169957} and there has been a lot if interest in developing optimal CNN architectures \cite{Krizhevsky:2012:ICD:2999134.2999257, 10.1007/978-3-319-10590-1_53, 43022, He2016DeepRL}. Current CNN architectures are  complex, with a high number of hyperparameters. In addition, the training sets for CNNs are large and this increases training times. Hence, we have a high number of trials, each trail with significant execution time. Decreasing the number of trials is critical. 

When applying WRS to our CNN optimization problem we consider the following hyperparameters:

\begin{itemize}
\item The number of convolution layers - an integer value in the set $\{3, 4, 5, 6\}$;
\item The number of fully-connected layers - an integer value in the set $\{1, 2, 3, 4\}$;
\item The number of output filters in each convolution layer - an integer value in the range $[100, 1024]$;
\item The number of neurons in each fully connected layer - an integer value in the range $[1024, 2048] $.
\end{itemize}

We generate each hyperparameter according to the uniform distribution and assess the performance of the model solely by the classification accuracy.

We use Keras \cite{chollet2015keras} to train and test the CNN for 300 trials - ten epochs each - on the CIFAR-10 \cite{cifar-10} dataset. We run our test on an IBM S822LC cluster with IBM POWER8 nodes, NVLink and NVidia Tesla P100 GPUs\footnote{http://www.cwu.edu/faculty/turing-cwu-supercomputer}. The CIFAR-10 dataset consists of 60000 $32 \times 32$ color images in 10 classes, with 6000 images per class. The data is split into  50000 training images and 10000 test images. We do not use data augmentation. 

The base architecture of the network is represented in Fig. \ref{fig:CNNArchitecture}. The model has between three and six $3 \times 3$ convolutional layers and between one and four fully connected layers. Both the convolutional and fully connected layers use ReLU \cite{Nair:2010:RLU:3104322.3104425} activation and the output layer uses softmax. We add one $2 \times 2$ MAX pooling layer with a dropout [25] of 0.25 for every two convolutional layers and use a dropout of 0.5 for the fully connected layers. We compare the results obtained by our WRS algorithm against the ones obtained by the RS, Nelder-Mead (NM), Particle Swarm (PS) \cite{kennedy95particle} and Sobol Sequences (SS) \cite{SOBOL1976236} implementations provided by Optunity \cite{web:Optunity}.

\begin{figure*}[htb!]
\centerline{\includegraphics[width=0.85\textwidth]{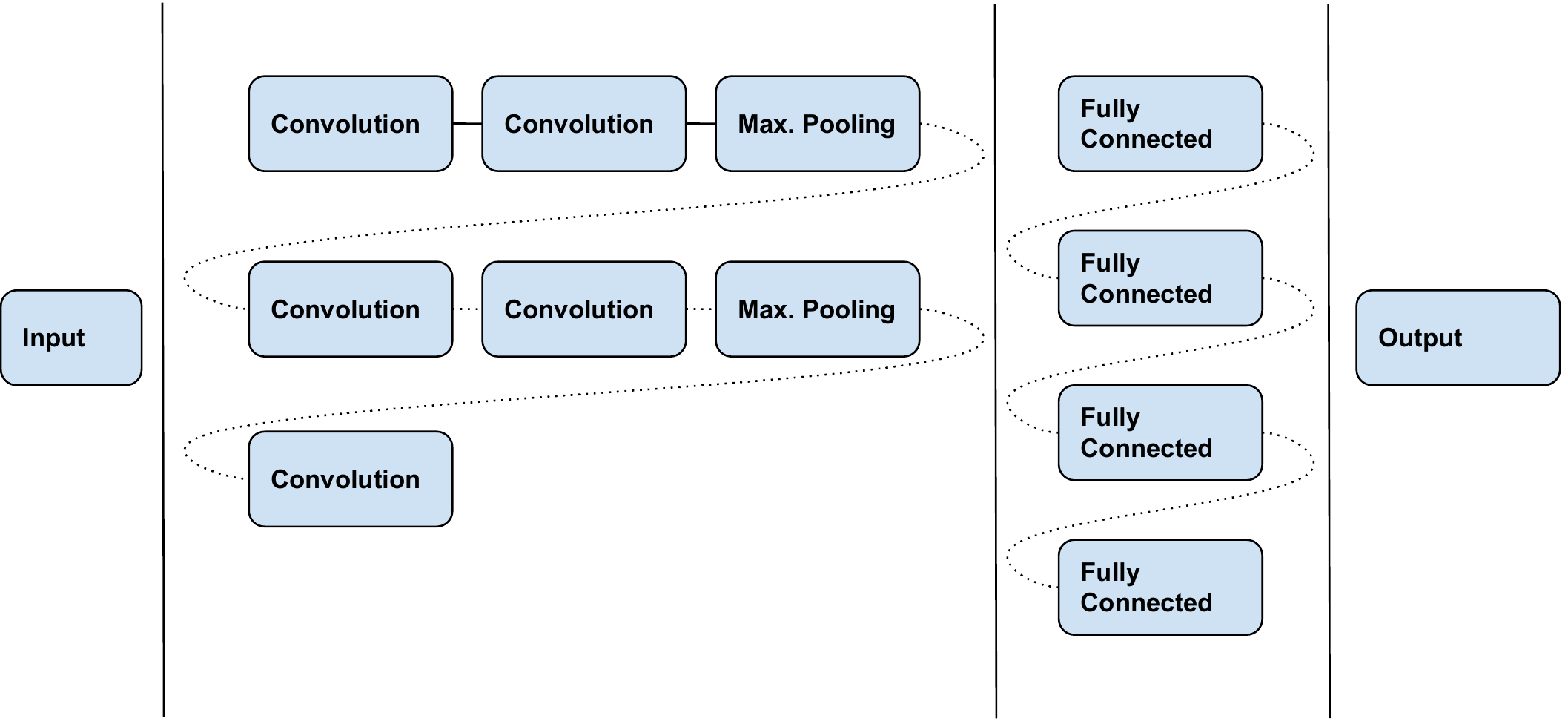}}
\caption{The CNN architecture}
\label{fig:CNNArchitecture}
\end{figure*}

After the first phase of the algorithm, which consists in running RS for $300/e = 110$ trials, we obtain the weights for each parameter. These values, along with the probabilities of change, are listed in Table \ref{table:weights-cnn}. After running fANOVA, the resulted most important three parameters are (in decreasing order of their weights):  the number of neurons in the first fully connected layer, the number of fully connected layers, and the number of convolutional layers. The weights of the other parameters are more than an order of magnitude smaller. Therefore, the second phase of WRS clearly favors the change in the first three most important parameters.

\begin{table}[h!]
\centering
\caption{Parameter weights and probabilities for CNN}
\resizebox{\textwidth}{!}{
\begin{tabular}{ |c|c|c|c|c|c|c|c|c|c|c|c| }
\hline
Convolutional & Fully Connected & & & & & & & & & &   \\
Layers & Layers & Conv 1 & Conv 2 & Conv 3 & Conv 4 & Conv 5 & Conv 6 & Full 1 & Full 2 & Full 3 & Full 4\\
\hline
7.4 & 11.85 & 0.51 & 0.79 & 1.62 & 0.73 & 2.26 & 1.26 & 26.28 & 0.87 & 3.22 & 1.75  \\
\hline
0.28 & 0.45 & 0.02 & 0.03 & 0.06 & 0.03 & 0.09 & 0.05 & 1.00 & 0.03 & 0.12 & 0.07  \\
\hline
\end{tabular}
}
\label{table:weights-cnn}
\end{table}

Fig. \ref{fig:CNN_RSvsWRS} shows the least squares five degree polynomial fit on the accuracy results obtained for each of the 300 trials using:  WRS - the solid line, RS, NM, PS, SS - the dashed lines. The trend of the WRS performance is similar to the one from Fig. \ref{fig:grievank-WRSvsRS}. The plot considers the actual values, reported at each iteration, instead of the local best in order to better reveal the variation of those values.

\begin{figure*}[htb!]
\centerline{\includegraphics[width=0.65\textwidth]{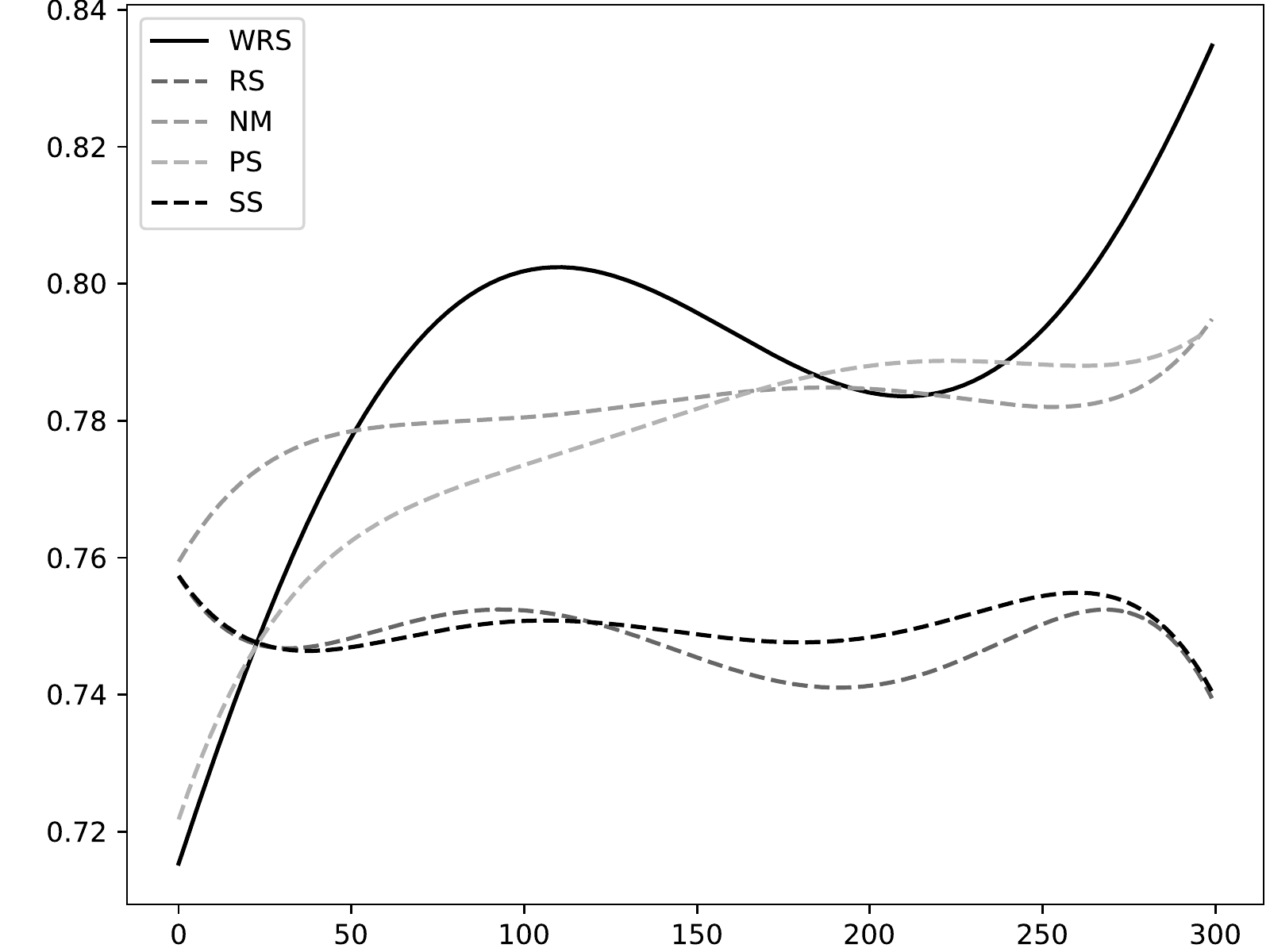}}
\caption{Least squares five degree polynomial fit on RS, NM, PS, SS vs. WRS accuracy for CIFAR-10 on 300 trials. The plot considers the values reported at each iteration}
\label{fig:CNN_RSvsWRS}
\end{figure*}

The best accuracy, as well as the average and standard deviation, across all 300 trials for all algorithms, are depicted in Table \ref{table:RSvsWRS-cnn}. WRS method outperforms all other considered methods (see Table \ref{table:RSvsWRS-cnn} and  Fig.  \ref{fig:CNN-WRSvsRS}). 

\begin{figure*}[htb!]
\centerline{\includegraphics[width=0.65\textwidth]{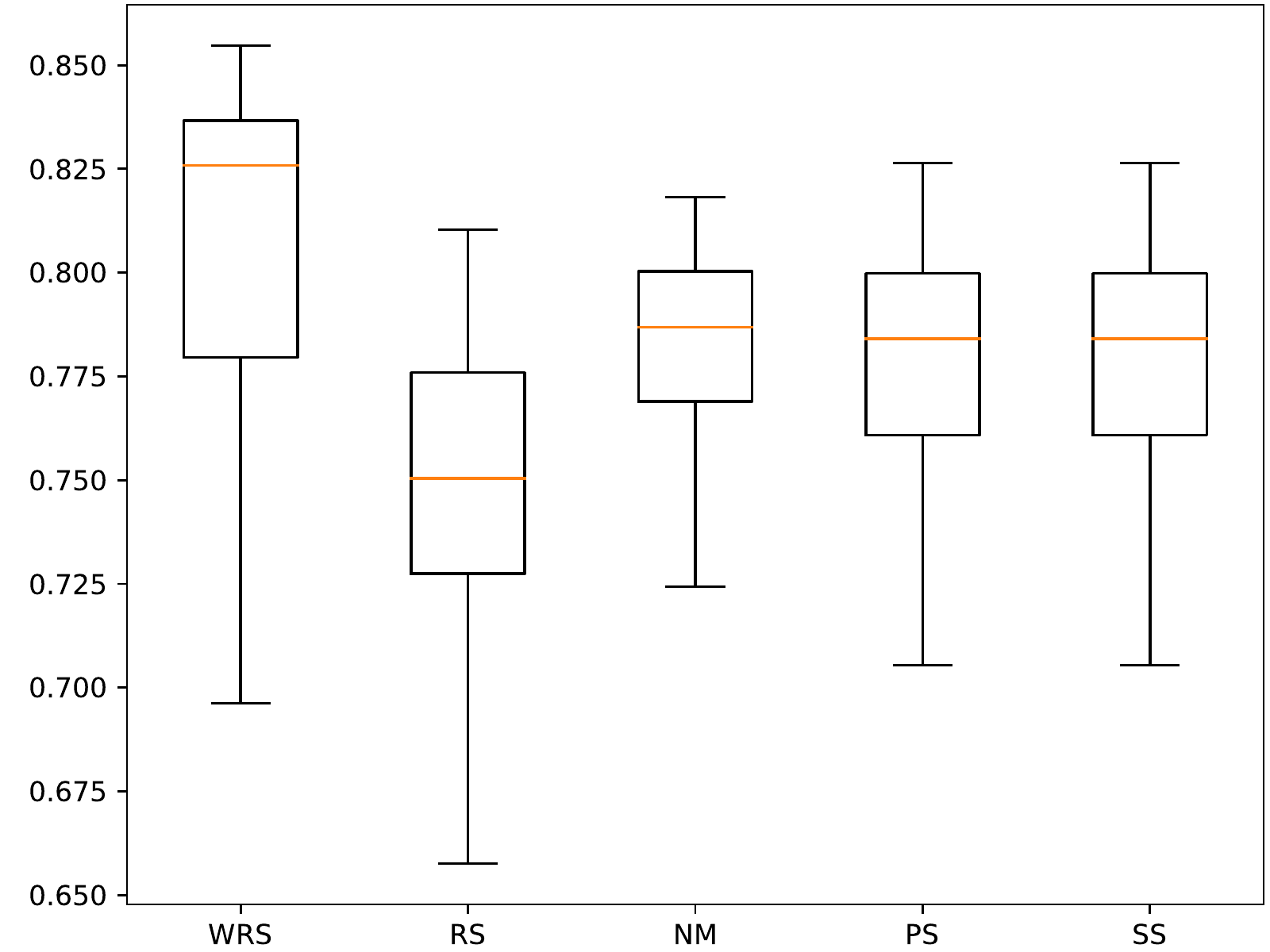}}
\caption{Performance of WRS, RS, NM, PS and SS for CNN optimization}
\label{fig:CNN-WRSvsRS}
\end{figure*}

\begin{table}[h!]
\caption{Algorithms' results for CNN accuracy on CIFAR-10}
\centering
\begin{tabular}{ |c|c|c|c| }
\hline
Optimizer & Best Result & Average & SD \\
\hline
WRS & \textbf{0.85} & \textbf{0.79} & 0.09  \\
\hline
RS & 0.81 & 0.75 & 0.04  \\
\hline
NM & 0.81 & 0.77 & 0.03  \\
\hline
PS & 0.83 & 0.78 & 0.03  \\
\hline
SS & 0.82 & 0.75 & 0.05  \\
\hline
\end{tabular}
\label{table:RSvsWRS-cnn}
\end{table}

Table \ref{table:RSvsWRS-cnn-architecture} shows the best found architecture by each algorithm. We observe that for the WRS and RS methods, the resulted architectures have only one fully connected layer and several convolutional layers (five for RS, six for WRS).

\begin{table}[h!]
\caption{Best identified CNN architectures on CIFAR-10}
\centering
\resizebox{\textwidth}{!}{
\begin{tabular}{ |c|c|c|c|c|c|c|c|c|c|c|c|c| }
\hline
Optimizer & Convolutional & Fully Connected  & & & & & & & & & &  \\
 & Layers & Layers & Conv 1 & Conv 2 & Conv 3 & Conv 4 & Conv 5 & Conv 6 & Full 1 & Full 2 & Full 3 & Full 4\\
\hline
WRS & 6 & 1 & 736 & 508 & 664 & 916 & 186 & 352 & 1229 & - & - & -  \\
\hline
RS & 5 & 1 & 876 & 114 & 892 & 696 & 617 & - & 1828 & - & - & -  \\
\hline
NM & 5 & 3 & 564 & 564 & 564 & 560 & 563 & - & 1529 & 1542 & 1542 & - \\
\hline
PS & 5 & 1 & 479 & 792 & 584 & 411 & 593 & - & 1379 & - & - & -  \\
\hline
SS & 5 & 2 & 402 & 933 & 750 & 997 & 777 & - & 1545 & 1268 & - & -  \\
\hline
\end{tabular}
}
\label{table:RSvsWRS-cnn-architecture}
\end{table}

Table \ref{table:WRS-heatmap} details the results obtained by WRS, showing the accuracy average and the standard deviation values for each combination: (number of fully connected layers, number of convolutional layers). Table \ref{table:WRS-trials} shows the number of trials performed by WRS for each of these combinations. 

We notice that the WRS algorithm favors one of the combinations, namely \{1, 6\}, and uses it for almost two thirds of the number of trials. It is important to mention that within the best 200 trials, only 10 sets of values contain a different combination than \{1, 6\}. This is either \{1, 5\} - seven times, or \{2, 5\} - three times. The first different combination than \{1, 6\} is at the 136-th position. In Table \ref{table:WRS-heatmap}, we observe that this combination also triggers the best results. 

This, together with the fact that WRS performs on average better than RS, validates our hypothesis that the probability that this combination of hyperparameters corresponds to the global optimum is higher than for any other combination. 

\begin{table}[h!]
\caption{WRS Accuracy Average and Standard Deviation. Row headings are numbers of fully connected layers while column headings are numbers of convolutional layers}
\centering
\begin{tabular}{ |c|c|c|c|c|c|c|c|c|c|c|c|c| }
\hline
FC &  &  &  &  \\
/C & 1 & 2 & 3 & 4\\
\hline
3 & 0.74 (0.02)& 0.70 (0.03)& 0.74 (0.01)& 0.69 (0.03) \\
\hline
4 & 0.78 (0.01)& 0.74 (0.03)& 0.74 (0.03)& 0.63 (0.07) \\
\hline
5 & 0.81 (0.02)& 0.80 (0.02)& 0.74 (0.07)& 0.65 (0.06) \\
\hline
6 & \textbf{0.82 (0.01)}& 0.76 (0.04)& 0.72 (0.09)& 0.39 (0.21) \\
\hline
\end{tabular}
\label{table:WRS-heatmap}
\end{table}

\begin{table}[h!]
\caption{WRS Number of Trials. Row headings are numbers of fully connected layers while column headings are numbers of convolutional layers}
\centering
\begin{tabular}{ |c|c|c|c|c|c|c|c|c|c|c|c|c| }
\hline
FC &  &  &  &  \\
/C & 1 & 2 & 3 & 4\\
\hline
3 & 4& 4& 4& 7\\
\hline
4 & 8& 3& 8& 9\\
\hline
5 & 9& 7& 9& 4\\
\hline
6 & \textbf{199}& 6& 10& 9\\
\hline
\end{tabular}
\label{table:WRS-trials}
\end{table}

\section{Conclusions} \label{conclusion}

We have introduced an improved version of RS, the WRS method. Within the same computational budget (i.e., for the same number of iterations), WRS converges on average faster than RS. The WRS algorithm yields better results both for the optimization of a well known difficult mathematical function and for a CNN hyperparameter optimization problem. There is little information required to be transferred between the  consecutive steps of the algorithm, as pointed out in the description of Algorithm \ref{alg:wrs_one_step}. This implies that the WRS  algorithm can be easily implemented in parallel. Since we made no assumptions on the objective function, our results can be generalized  to other optimization problems defined on a discrete domain. We plan to test out algorithm on other classes of optimization problems, in particular on the optimization of various machine learning algorithms. We also plan to compare the results obtained with WRS with other more complicated optimization techniques, especially from the very promising area of Bayesian optimization.

\end{document}